\def\year{2019}\relax
\documentclass[letterpaper]{article} 
\usepackage{aaai19}  
\usepackage{times}  
\usepackage{helvet}  
\usepackage{courier}  
\usepackage{url}  
\usepackage{graphicx}  
\usepackage{subcaption}
\usepackage[ruled,vlined]{algorithm2e}
\usepackage[noend]{algpseudocode}
\usepackage{amsthm,amsmath,amssymb}
\usepackage{bm}
\usepackage{latexsym}
\usepackage{float}
\usepackage{color}
\usepackage{pifont}
\newcommand{\argmin}{\mathop{\rm arg~min}}
\newcommand{\supremum}{\mathop{\rm sup}\limits}
\newcommand{\expectation}{{\mathbb{E}}}
\newcommand{\citet}[1]{\citeauthor{#1}~\shortcite{#1}}

\theoremstyle{definition}
\newtheorem{theorem}{Theorem}
\newtheorem*{theorem*}{Theorem}
\newtheorem*{proof*}{proof}
\newtheorem{definition}[theorem]{Definition}
\newtheorem*{definition*}{Definition}
\newtheorem{lemma}[theorem]{Lemma}
\newtheorem{lemma*}{Lemma}
\newtheorem{proposition}[theorem]{Proposition}
\newtheorem*{proposition*}{Proposition}
\newtheorem{corollary}[theorem]{Corollary}
\newcommand{\cmark}{\ding{51}}

\begin{document}
%
\title{Unsupervised Domain Adaptation Based on Source-guided Discrepancy \thanks{This is a longer version of Kuroki et al.
(2019).}}
\author{Seiichi Kuroki$^{\dag\ddag}$, Nontawat Charoenphakdee$^\dag$, Han Bao$^{\dag\ddag}$, \\ \bf \Large Junya Honda$^{\dag\ddag}$, Issei Sato$^{\dag\ddag}$, Masashi Sugiyama$^{\ddag\dag}$\\
$^\dag$The University of Tokyo, Tokyo, Japan\\
$^\ddag$RIKEN, Tokyo, Japan\\
}
\maketitle
\begin{abstract}
\emph{Unsupervised domain adaptation} is the problem setting where data generating distributions in the source and target domains are different and labels in the target domain are unavailable.
An important question in unsupervised domain adaptation is how to measure the difference between the source and target domains. Existing discrepancy measures for unsupervised domain adaptation either require high computation costs or have no theoretical guarantee. To mitigate these problems, this paper proposes a novel discrepancy measure called \emph{source-guided discrepancy (S-disc)}, which exploits labels in the source domain unlike the existing ones.
As a consequence, S-disc can be computed efficiently with a finite-sample convergence guarantee. In addition, it is shown that S-disc can provide a tighter generalization error bound than the one based on an existing discrepancy measure. Finally, experimental results demonstrate the advantages of S-disc over the existing discrepancy measures. 
\end{abstract}

\section{Introduction}
In the conventional supervised learning framework, we often assume that
the training and test
distributions are the same. However, this assumption may not hold in many practical applications such as natural language processing~\cite{glorot2011domain}, speech recognition~\cite{sun2017unsupervised}, and computer vision~\cite{saito2017asymmetric}. For instance, a personalized spam filter can be trained from the emails of all available users, but the training data may not represent the emails of the target user. 
Such scenarios can be formulated in the framework of \emph{domain adaptation}, which has been studied extensively~\cite{ben2007analysis,mansour2009domain,zhang2012generalization,IEEE-KDE:Pan+Yang:2010,book:Sugiyama+Kawanabe:2012}.
An important challenge in domain adaptation is to find a classifier for a label-scarce target domain by exploiting a label-rich source domain. In particular, there are cases where we only have access to labeled data from the source domain and unlabeled data from the target domain, because annotating labels in the target domain is often time-consuming and expensive~\cite{saito2017maximum}. This problem setting is called \emph{unsupervised domain adaptation}~\cite{ben2007analysis}, which is our interest in this paper. 

Since domain adaptation cannot be performed effectively if
the source and target domains
are too different, an important topic to be addressed is how to measure the difference, or {\it discrepancy}, between the two domains. 
Many discrepancy measures have been used in previous studies such as the  maximum mean discrepancy~\cite{huang2007correcting}, Kullback-Leibler divergence~\cite{sugiyama2008direct}, R\'{e}nyi divergence~\cite{mansour2009multiple}, and Wasserstein distance~\cite{courty2017joint}.

Apart from the discrepancy measures described above, \citet{ben2007analysis} proposed a discrepancy measure for binary classification that explicitly takes a hypothesis class into account. They showed that their discrepancy measure leads to a tighter generalization error bound than the $L_1$
distance, which
does not use information of the hypothesis class. Following this line of research, \citet{mansour2009domain} generalized the discrepancy measure of \citet{ben2007analysis} to arbitrary loss functions, and \citet{germain2013pac} provided a PAC-Bayesian analysis based on the average disagreement of a set of hypotheses. 
These results provided a theoretical foundation for various applications of
domain adaptation based on estimation of source-target discrepancy, including sentiment analysis~\cite{glorot2011domain}, source selection~\cite{bhatt2016multi}, and multi-source domain adaptation~\cite{zhao2018multiple}. 
However, their discrepancy measure considers the worst pair of hypotheses to bound the maximum gap in the loss between the domains, which may lead to a loose generalization error bound and expensive computation costs. 
\citet{zhang2012generalization} and~\citet{mohri2012new} analyzed another discrepancy measure with a generalization error bound. 
Nevertheless, this discrepancy measure cannot be computed without labels in the target domain and therefore is not suitable for unsupervised domain adaptation. 

To alleviate the limitations of the existing discrepancy measures for domain adaptation, we propose a novel discrepancy measure called \emph{source-guided discrepancy (S-disc)}.
By incorporating labels in the source domain, S-disc
provides a tighter generalization error bound than the previously proposed discrepancy measure by \citet{mansour2009domain}.
Furthermore, we provide an estimator of S-disc for binary classification that can be computed efficiently.
We also establish consistency and analyze the convergence rate of the proposed S-disc estimator.

Our main contributions are as follows.
\begin{itemize}
\item We propose a novel discrepancy measure called \emph{source-guided discrepancy (S-disc)}, which uses labels in the source domain to measure the difference between the two domains for unsupervised domain adaptation~(Definition~\ref{def:s-disc}).
\item We propose an efficient algorithm for estimating S-disc for the 0-1 loss~(Algorithm~\ref{alg1}).
\item We show the consistency and elucidate the convergence rate of the estimator of S-disc~(Theorem~\ref{thm:sipm-deviation}).
\item We derive a generalization error bound in the target domain based on S-disc~(Theorem~\ref{thm:generalization-bound}), which is tighter than the existing bound provided by \citet{mansour2009domain}. In addition, we derive a generalization error bound for a finite-sample case~(Theorem~\ref{thm:generalization-bound-with-s-disc}).
\item We demonstrate the effectiveness of S-disc for unsupervised domain adaptation in experiments.
\end{itemize}

\section{Problem Setting and Notation}
In this section, we formulate the unsupervised domain adaptation problem.
Let $\mathcal{X}$ be the input space and $\mathcal{Y}$ be the output space, which is $\{+1,-1\}$ in binary classification. We define a domain as a pair $(P_{\mathrm{D}}, f_{\mathrm{D}})$, where $P_{\mathrm{D}}$ is an input distribution on $\mathcal{X}$ and $f_{\mathrm{D}}: \mathcal{X} \rightarrow \mathcal{Y}$ is a labeling function.
In domain adaptation, we denote the source domain and the target domain as ($P_{\mathrm{S}}$,$f_{\mathrm{S}}$) and ($P_{\mathrm{T}}$,$f_{\mathrm{T}}$), respectively. 
Unlike conventional supervised learning, we focus on the case where ($P_{\mathrm{T}}$,$f_{\mathrm{T}}$) differs from ($P_{\mathrm{S}}$,$f_{\mathrm{S}}$). 
In unsupervised domain adaptation, we are given the following data:
\begin{itemize}
\item Unlabeled data in the target domain: $\mathcal{T} = \{ x_i^{\mathrm{T}}\}_{i=1}^{n_{\mathrm{T}}}$.
\item Labeled data in the source domain: $\mathcal{S} = \{ (x_j^{\mathrm{S}},y_j^{\mathrm{S}})\}_{j=1}^{n_{\mathrm{S}}}$.
\end{itemize}
For simplicity, we denote the input features from the source domain $\{ x_j^{\mathrm{S}}\}_{j=1}^{n_{\mathrm{S}}}$ as $\mathcal{S}_{\mathcal{X}}$ and the empirical distribution corresponding to $P_{\mathrm{T}}$ (resp.~$P_{\mathrm{S}}$) as $\widehat{P}_{\mathrm{T}}$ (resp.~$\widehat{P}_{\mathrm{S}}$).

We denote a loss function $\ell$: $\mathcal{\mathbb{R}} \times \mathcal{\mathbb{R}} \rightarrow \mathbb{R}_{+}$.
For example,
the 0-1 loss is given as $\ell_{01}(y,y^{\prime}) = (1- \mathrm{sign}(yy^{\prime}))/2$.
The expected loss for functions $h,h^{\prime}: \mathcal{X} \rightarrow \mathbb{R}$ and distribution $P_{\mathrm{D}}$ over $\mathcal{X}$ is denoted by $R_{\mathrm{D}}^{\ell}(h,h^{\prime}) =\expectation_{x \sim P_{\mathrm{D}}} [\ell(h(x), h^{\prime}(x))]$.
We also denote an empirical risk as $\widehat{R}_{\mathrm{D}}^{\ell}(h,h^{\prime}) =\expectation_{x \sim \widehat{P}_{\mathrm{D}}} [\ell(h(x), h^{\prime}(x))]$. 
In addition, we define the true risk minimizer and the empirical risk minimizer in a certain domain ($P_{\mathrm{D}}$,$f_{\mathrm{D}}$) in a hypothesis class $\mathcal{H}$ as $h_{\mathrm{D}}^{*}=\argmin_{h \in \mathcal{H}}R_{\mathrm{D}}^{\ell}(h,f_{\mathrm{D}})$ and $\widehat{h}_{\mathrm{D}}=\argmin_{h \in \mathcal{H}}\widehat{R}_{\mathrm{D}}^{\ell}(h,f_{\mathrm{D}})$, respectively.
Here, note that the risk minimizer $h_{\mathrm{D}}^{*}$ is not necessarily equal to the labeling function $f_{\mathrm{D}}$ as we consider a restricted hypothesis class.

The goal in domain adaptation is to find a hypothesis $h$ out of a hypothesis class $\mathcal{H}$ so that it gives as small expected loss as possible in the target domain defined by 
\begin{equation*}
    R_{\mathrm{T}}^{\ell}(h,f_{\mathrm{T}}) = \expectation_{x \sim P_{\mathrm{T}}} [\ell(h(x), f_{\mathrm{T}}(x))].
\end{equation*}

\section{Related Work}
In unsupervised domain adaptation, it is essential to measure the difference between the source and target domains,
because it might degrade the performance to use data collected from a source domain that is far from the target domain~\cite{IEEE-KDE:Pan+Yang:2010}.
Since we cannot access labels from the target domain, it is impossible to measure the difference between the two domains on the basis of the \emph{output space}.
One way is to measure the difference between probability distributions in terms of the \emph{input space}.

We will review such discrepancy measures in this section. First, we introduce the discrepancy measure proposed by \citet{mansour2009domain}, which is defined as
\begin{equation}
\mathrm{disc}_{\mathcal{H}}^{\ell}(P_{\mathrm{T}},P_{\mathrm{S}}) = \supremum_{h,h^{\prime} \in \mathcal{H}}\left| R_{\mathrm{T}}^{\ell}(h,h^{\prime})-R_{\mathrm{S}}^{\ell}(h,h^{\prime})\right|. \label{def:x-disc}
\end{equation}
We call this discrepancy measure $\mathcal{X}$-disc as it does not require labels from the source domain but input features. Note that $\mathcal{X}$-disc takes the hypothesis class $\mathcal{H}$ into account.

\citet{mansour2009domain} showed that the following inequalities hold for any $h, h^{\prime} \in \mathcal{H}$:
\begin{align*}
\left| R_{\mathrm{T}}^{\ell}(h,h^{\prime})-R_{\mathrm{S}}^{\ell}(h,h^{\prime})\right| &\leq \mathrm{disc}_{\mathcal{H}}^{\ell}(P_{\mathrm{T}},P_{\mathrm{S}})\\ &\leq M\cdot L_1(P_{\mathrm{T}},P_{\mathrm{S}}),
\end{align*}
where $M >0$ is a constant and $L_1(\cdot,\cdot)$ is the $L_1$ distance over distributions. 
Therefore, a tighter bound for the difference in the expected loss between the two domains can be obtained by considering a hypothesis class $\mathcal{H}$.
However, a drawback of $\mathcal{X}$-disc is that it considers the worst pair of hypotheses as can be seen from \eqref{def:x-disc}.
This may lead to a loose generalization error bound as we will show later, and also an intractable computation cost for empirical estimation. \citet{ben2007analysis} provided a computationally efficient proxy of $\mathcal{X}$-disc for the 0-1 loss defined as follows:
\begin{equation*}
d_{\mathcal{H}}(P_{\mathrm{T}},P_{\mathrm{S}}) = \supremum_{h \in \mathcal{H}}\left| R_{\mathrm{T}}^{\ell_{01}}(h,1)-R_{\mathrm{S}}^{\ell_{01}}(h,1)\right|.
\end{equation*} 
Although $d_{\mathcal{H}}$ can be computed efficiently, there is no learning guarantee.

There is another variant of
$\mathcal{X}$-disc called the generalized discrepancy~\cite{cortes2015adaptation}.
However, the theoretical analysis therein is only applicable to the regression task and not suitable for binary classification unlike the analysis for $\mathcal{X}$-disc.

Another discrepancy measure $\mathcal{Y}$-disc which also takes $\mathcal{H}$ into account is as follows:
\begin{equation*}
    {\mathcal{Y}\text-\mathrm{disc}}_\mathcal{H}^{\ell}(P_{\mathrm{T}},P_{\mathrm{S}}) = \supremum_{h \in \mathcal{H}}|R_{\mathrm{T}}^{\ell}(h,f_{\mathrm{T}})-R_{\mathrm{S}}^{\ell}(h,f_{\mathrm{S}})|.
\end{equation*}
While $\mathcal{Y}$-disc has been proposed to provide a tighter generalization error bound than $\mathcal{X}$-disc~\cite{mohri2012new,zhang2012generalization}, it requires the labeling function $f_{\mathrm{T}}$ in the test domain which cannot be estimated in unsupervised domain adaptation in general.

\section{Source-guided Discrepancy (S-disc)}
In this section, we propose a novel discrepancy measure called \emph{source-guided discrepancy~(S-disc)} to mitigate the limitations of the existing
discrepancy measures. Later, we will show that S-disc can provide a tighter generalization bound and can be computed efficiently compared with the existing measures. 

We define S-disc as follows, which is obtained by fixing $h^{\prime}$ in the definition~\eqref{def:x-disc} of $\mathcal{X}$-disc to the true risk minimizer $h_{\mathrm{S}}^{*}=\argmin_{h \in \mathcal{H}}R_{\mathrm{S}}^{\ell}(h,f_{\mathrm{S}})$ in the source domain.

\begin{table*}[t]
    \centering
    \caption{
Comparison of S-disc with the existing discrepancy measures for unsupervised domain adaptation in binary classification.
        We assume the hypothesis class $\mathcal{H}$ satisfies~\eqref{assump:rademacher}.
       We
       consider
       the hinge loss for S-disc, $\mathcal{X}$-disc, and $d_{\mathcal{H}}$.
        The computational complexity of S-disc and $d_{\mathcal{H}}$ are based on the empirical hinge loss minimization, which is solved with the kernel support vector machine by the SMO algorithm~\cite{Platt98sequentialminimal}.
        The computational complexity of $\mathcal{X}$-disc is that based on SDP relaxation by the ellipsoid method~\cite{bubeck2015convex} given in Appendix.
        ${\mathcal{Y}\text-\mathrm{disc}}$ is not computable with unlabeled target data. 
    }
    \begin{tabular}{|c||c|c|c|c|} \hline
        & S-disc (proposed) & $d_{\mathcal{H}}$ & $\mathcal{X}$-disc & ${\mathcal{Y}\text-\mathrm{disc}}$ \\ \hline \hline
        taking $\mathcal{H}$ into account & \cmark & \cmark & \cmark & \cmark \\ \hline
         convergence rate & $\mathcal{O}_p\left({n_{\mathrm{T}}}^{-\frac{1}{2}} + {n_{\mathrm{S}}}^{-\frac{1}{2}}\right)$ & $\mathcal{O}_p\left({n_{\mathrm{T}}}^{-\frac{1}{2}} + {n_{\mathrm{S}}}^{-\frac{1}{2}}\right)$ & $\mathcal{O}_p\left({n_{\mathrm{T}}}^{-\frac{1}{2}} + {n_{\mathrm{S}}}^{-\frac{1}{2}}\right)$ & N/A \\ \hline
        target generalization error bound & \cmark & N/A & \cmark & \cmark \\ \hline
        computational complexity & $O\left((n_{\mathrm{T}} + n_{\mathrm{S}})^3\right)$ & $O\left((n_{\mathrm{T}} + n_{\mathrm{S}})^3\right)$ & $O\left((n_{\mathrm{T}} + n_{\mathrm{S}}+d)^8\right)$ & N/A \\ \hline
    \end{tabular}
    \label{tab:comparison}
\end{table*}

\begin{definition}[Source-guided discrepancy]
Let $\mathcal{H}$ be a hypothesis class and let $\ell: \mathbb{R} \times \mathbb{R} \rightarrow \mathbb{R}_{+}$ be a loss function.
S-disc
between two distributions
$P_{\mathrm{D}_1}$ and $P_{\mathrm{D}_2}$ is defined as   
\begin{align}
\varsigma_{\mathcal{H}}^{\ell}(P_{\mathrm{\mathrm{D}_1}},P_{\mathrm{D}_2})
= \supremum_{h \in \mathcal{H}} \left| R_{\mathrm{\mathrm{D}_1}}^{\ell}(h,h_{\mathrm{S}}^*)- R_{\mathrm{D}_2}^{\ell}(h,h_{\mathrm{S}}^*)\right|.
\label{def:s-disc}
\end{align}
\end{definition}

S-disc satisfies the triangular inequality, i.e., $\varsigma_{\mathcal{H}}^{\ell}(P_{\mathrm{D}_1},P_{\mathrm{D}_2}) \leq \varsigma_{\mathcal{H}}^{\ell}(P_{\mathrm{D}_1},P_{\mathrm{D}_3}) + \varsigma_{\mathcal{H}}^{\ell}(P_{\mathrm{D}_3},P_{\mathrm{D}_2}) $, and is symmetric, i.e., $\varsigma_{\mathcal{H}}^{\ell}(P_{\mathrm{D}_1},P_{\mathrm{D}_2}) = \varsigma_{\mathcal{H}}^{\ell}(P_{\mathrm{D}_2},P_{\mathrm{D}_1})$. However, in general, S-disc is not a distance as we may have $\varsigma_{\mathcal{H}}^{\ell}(P_{\mathrm{D}_1},P_{\mathrm{D}_2})=0$ for some $P_{\mathrm{D}_1} \neq P_{\mathrm{D}_2}$.

S-disc has the following three advantages over existing discrepancy measures.
First, the computation cost of S-disc is low since we do not need to consider a pair of hypotheses
unlike $\mathcal{X}$-disc as we can see from
\eqref{def:x-disc} and \eqref{def:s-disc}.
Second, S-disc leads to a tighter bound
as discussed below.
\citet{mansour2009domain} derived a generalization error bound by bounding the difference 
$\left| R_{\mathrm{T}}^{\ell}(h,h_{\mathrm{S}}^{*})- R_{\mathrm{S}}^{\ell}(h,h_{\mathrm{S}}^{*})\right|$ with $\mathcal{X}$-disc.
On the other hand, 
for any $h \in \mathcal{H}$, it is easy to see that from the definition of S-disc
\begin{align} 
\left| R_{\mathrm{T}}^{\ell}(h,h_{\mathrm{S}}^{*})- R_{\mathrm{S}}^{\ell}(h,h_{\mathrm{S}}^{*})\right| &\leq \varsigma_{\mathcal{H}}^{\ell}(P_{\mathrm{T}},P_{\mathrm{S}})\notag\\ &\leq \mathrm{disc}_{\mathcal{H}}^{\ell}(P_{\mathrm{T}},P_{\mathrm{S}}),  \label{eq:1}
\end{align}
which implies that S-disc gives a tighter generalization error bound than $\mathcal{X}$-disc. 
Third, since $h_{\mathrm{S}}^{*}$ is estimated by labeled data from the source domain, we can compute S-disc without labels from the target domain unlike ${\mathcal{Y}\text-\mathrm{disc}}$. 
For these reasons, S-disc is more suitable for unsupervised domain adaptation
than the existing discrepancy measures.
A comparison of S-disc with the existing discrepancy measures is summarized in Table~\ref{tab:comparison}.

\section{S-disc Estimation for the 0-1 Loss}
Here we consider the task of binary classification
with the 0-1 loss, where the output space $\mathcal{Y}$ is
$\{+1, -1\}$.
The following theorem states that S-disc estimation can be reduced to a cost-sensitive classification problem. We consider a symmetric hypothesis class $\mathcal{H}$, which is closed under negation, i.e., for any $h \in \mathcal{H}$, $-h$ is contained in $\mathcal{H}$.
\begin{theorem} \label{th:01lossest}
\label{thm:sipm-estimation}
For the 0-1 loss and a symmetric hypothesis class $\mathcal{H}$,
the following equality holds:
\begin{align*}
\varsigma_{\mathcal{H}}^{\ell_{01}}(\widehat{P}_{\mathrm{T}},\widehat{P}_{\mathrm{S}})= 1 - \min_{h \in \mathcal{H}}J_{\ell_{01}}(h),
\end{align*}
where $J_{\ell}(h)$ is defined as
\begin{align*}
J_{\ell}(h)&=\frac{1}{n_{\mathrm{S}}}\sum_{j=1}^{n_{\mathrm{S}}}\ell(h(x_j^{\mathrm{S}}),h_{\mathrm{S}}^*(x_j^{\mathrm{S}}))\\
&\quad+ \frac{1}{n_{\mathrm{T}}}\sum_{i=1}^{n_{\mathrm{T}}}\ell(h(x_i^{\mathrm{T}}), -h_{\mathrm{S}}^*(x_i^{\mathrm{T}})).
\end{align*}
\end{theorem}
The proof of Theorem~\ref{th:01lossest} is given in Appendix.
Note that the minimization of $J_{\ell_{01}}(h)$ corresponds to empirical risk minimization for $\mathcal{S}^* = \{(x, {h}_{\mathrm{S}}^*(x)) \mid x \in \mathcal{S}_{\mathcal{X}}\}$ and $\mathcal{T}^* = \{(x, -{h}_{\mathrm{S}}^*(x)) \mid x \in \mathcal{T}\}$.
Therefore, Theorem~\ref{thm:sipm-estimation} naturally suggests a three-step algorithm illustrated in Algorithm~\ref{alg1}, where $h_{\mathrm{S}}^*$ is replaced with $\widehat{h}_{\mathrm{S}}$.
This allows us to compute S-disc by minimizer $h^{\prime\prime}$ of the cost-sensitive risk $J_{\ell_{01}}(h)$, which weights the pseudo-labeled data $\widetilde{\mathcal{S}}$ and $\widetilde{\mathcal{T}}$ with costs $1/n_\mathrm{S}$ and $1/n_\mathrm{T}$, respectively.

Since minimization of the 0-1 loss is computationally hard \cite{ben2003difficulty,feldman2012agnostic}, we use a surrogate loss such as the hinge loss $\ell_{\text{hinge}}(y,y') = \max(0, 1-yy')$~\cite{bartlett2006convexity}.
Note that the 0-1 loss is used for calculating S-disc in the final step, i.e., $1-J_{\ell_{01}}(h^{\prime\prime})$, while a surrogate loss is only used to train a classifier.
Hinge loss minimization with the sequential minimal optimization (SMO) algorithm requires $O((n_{\mathrm{T}} + n_{\mathrm{S}})^3)$ for the entire algorithm~\cite{Platt98sequentialminimal}.
This low computation cost has
a big advantage over
$\mathcal{X}$-disc (see Table~\ref{tab:comparison}).
\begin{algorithm}[t]
\caption{S-disc Estimation for the 0-1 Loss}
\label{alg1}
\SetKwInput{Input}{Input}\SetKwInOut{Output}{Output}
\SetKw{KwInit}{Initialize}
\Input{labeled source data $\mathcal{S}$, unlabeled target data $\mathcal{T}$, surrogate loss $\ell_{\text{sur}}$, hypothesis class $\mathcal{H}$.}
\Output{$\varsigma_{\mathcal{H}}^{\ell}(\widehat{P}_{\mathrm{T}},\widehat{P}_{\mathrm{S}})$.}
\textbf{Source learning}:\\
Learn a classifier $\widehat{h}_{\mathrm{S}}$ using labeled source data $\mathcal{S}_{\mathcal{X}}$.\\
\textbf{Pseudo labeling}:\\
\begin{itemize}
\item $\widetilde{\mathcal{S}} =  \{(x, \mathrm{sign}\circ\widehat{h}_{\mathrm{S}}(x)) \mid x \in \mathcal{S}_{\mathcal{X}}\}$,
\item $\widetilde{\mathcal{T}} =  \{(x, - \mathrm{sign}\circ\widehat{h}_{\mathrm{S}}(x)) \mid x \in \mathcal{T}\}$.
\end{itemize}
\textbf{Cost sensitive learning from pseudo labeled data}\\
Learn another classifier $h^{\prime\prime} \in \mathcal{H}$ using $\widetilde{\mathcal{S}}$ and $\widetilde{\mathcal{T}}$ to minimize the surrogate cost-sensitive risk $J_{\ell_{\text{sur}}}$. \\
\Return $\varsigma_{\mathcal{H}}^{\ell}(\widehat{P}_{\mathrm{T}},\widehat{P}_{\mathrm{S}})$ =  $1-J_{\ell_{01}}(h^{\prime\prime})$.
\\
\end{algorithm}

\section{Theoretical Analysis}
In this section, we show that S-disc can be estimated from finite data with the consistency guarantee.
After that, we show that our proposed S-disc is useful for deriving a tighter generalization error bound than $\mathcal{X}$-disc. 
To derive theoretical results, the \emph{Rademacher complexity} is used, which captures the complexity of a set of functions by measuring the capability of a hypothesis class to correlate with the random noise.

\begin{definition}[Rademacher complexity~\cite{bartlett2002rademacher}]
Let $\mathcal{H}$ be a set of real-valued functions defined over a set $\mathcal{X}$. Given a sample $(x_1,\dots,x_m) \in \mathcal{X}^m$ independently and identically drawn from a distribution $\mu$, the Rademacher complexity of $\mathcal{H}$ is defined as
\begin{equation*}
\mathfrak{R}_{\mu,m}(\mathcal{H}) =  \expectation_{x_1,\dots,x_m}\expectation_{\sigma}\left[ \supremum_{h \in \mathcal{H}} \frac{1}{m}\sum_{i=1}^{m} \sigma_{i} h(x_i) \right],
\end{equation*}
where the inner expectation is taken over $\sigma = (\sigma_1, \dots, \sigma_{m})$ which are mutually independent uniform random variables taking values in $\{+1, -1\}$.
\end{definition}

Hereafter, we use the following notation:
\begin{itemize}
    \item $\mathcal{H}\otimes\mathcal{H} := \{x \mapsto h(x)\cdot h'(x)  \mid h,h' \in \mathcal{H}\}$,
    \item $\ell\circ(\mathcal{H}\otimes\mathcal{H}) := \{x \mapsto \ell\left(h(x),h'(x)\right) \mid h,h' \in \mathcal{H}\}$.
\end{itemize}

\subsection{Consistency of S-disc}
In this section, we show the estimator $\varsigma_{\mathcal{H}}^\ell(\widehat{P}_{\mathrm{T}},\widehat{P}_{\mathrm{S}})$ converges to the true S-disc $\varsigma_{\mathcal{H}}^\ell(P_{\mathrm{T}},P_{\mathrm{S}})$ as the numbers of samples $n_{\mathrm{T}}$ and $n_{\mathrm{S}}$ increase.
The following theorem gives the deviation of the empirical S-disc estimator,
which is a general result that does not depend on the specific choice of the loss $\ell$ and the hypothesis class $\mathcal{H}$.
\begin{theorem}
\label{thm:sipm-deviation}
Assume the loss function $\ell$ is bounded from above by $M > 0$.
For any $\delta \in (0, 1)$, with probability at least $1 - \delta$,
\begin{align}
&\left|\varsigma_{\mathcal{H}}^{\ell}(\widehat{P}_{\mathrm{T}},\widehat{P}_{\mathrm{S}}) - \varsigma_{\mathcal{H}}^{\ell}(P_{\mathrm{T}},P_{\mathrm{S}}) \right| \nonumber \\
&\qquad \le 2\mathfrak{R}_{P_{\mathrm{T}},n_{\mathrm{T}}}(\ell\circ(\mathcal{H}\otimes\mathcal{H})) + 2\mathfrak{R}_{P_{\mathrm{S}},n_{\mathrm{S}}}(\ell\circ(\mathcal{H}\otimes\mathcal{H})) \nonumber \\
&\qquad \quad + M\sqrt{\frac{\log\frac{4}{\delta}}{2n_{\mathrm{T}}}} + M\sqrt{\frac{\log\frac{4}{\delta}}{2n_{\mathrm{S}}}}.
\nonumber
\end{align}
\end{theorem}
The proof of this theorem is given in Appendix.

Theorem~\ref{thm:sipm-deviation} guarantees the consistency of $\varsigma_{\mathcal{H}}^{\ell}(\widehat{P}_{\mathrm{T}},\widehat{P}_{\mathrm{S}})$
under the condition that
$\mathfrak{R}_{P_{\mathrm{T}},n_{\mathrm{T}}}(\ell\circ(\mathcal{H}\otimes\mathcal{H}))$ and $\mathfrak{R}_{P_{\mathrm{S}},n_{\mathrm{S}}}(\ell\circ(\mathcal{H}\otimes\mathcal{H}))$ are well-controlled.

To derive a specific convergence rate,
we consider an assumption given by
\begin{align}
    \!\mathfrak{R}_{P_{\mathrm{T}},n_{\mathrm{T}}}(\mathcal{H}\otimes\mathcal{H}) \le \frac{\mathfrak{C}_{\mathcal{H}\otimes\mathcal{H}}}{\sqrt{n_{\mathrm{T}}}}, \;
    \mathfrak{R}_{P_{\mathrm{S}},n_{\mathrm{S}}}(\mathcal{H}\otimes\mathcal{H}) \le \frac{\mathfrak{C}_{\mathcal{H}\otimes\mathcal{H}}}{\sqrt{n_{\mathrm{S}}}},
    \label{assump:rademacher}
\end{align}
for some constant $\mathfrak{C}_{\mathcal{H}\otimes\mathcal{H}} > 0$ depending only on the hypothesis class $\mathcal{H}$.
Lemma \ref{lem:rademacher-bound} below
shows that this assumption is naturally satisfied
in the linear-in-parameter model.
\begin{lemma}
    \label{lem:rademacher-bound}
    Let $\mathcal{H}$ be the linear-in-parameter model class, i.e., $\mathcal{H} := \{x \mapsto w^\top\phi(x) \mid w \in \mathbb{R}^d, \; \|w\|_2 \le \Lambda \}$ for fixed basis functions $\phi: \mathcal{X} \to \mathbb{R}^d$ satisfying $\|\phi\|_\infty \le D_\phi$.
    Then, for any distribution $\mu$ over $\mathcal{X}$ and $m \in \mathbb{N}$,
    $$
        \mathfrak{R}_{\mu, m}(\mathcal{H}\otimes\mathcal{H}) \le \frac{\Lambda^2D_\phi^2}{\sqrt{m}}.
    $$
\end{lemma}

The proof of Lemma~\ref{lem:rademacher-bound} is given in Appendix.
Subsequently, we derive the convergence rate bound for the 0-1 loss.

\begin{corollary}
    \label{cor:01-sipm-rate}
When we consider
$\ell= \ell_{01}$,
it holds for any $\delta\in(0,1)$
that, with probability at least $1 - \delta$,
\begin{align*}
        &\left|\varsigma_{\mathcal{H}}^{\ell}(\widehat{P}_{\mathrm{T}},\widehat{P}_{\mathrm{S}}) - \varsigma_{\mathcal{H}}^{\ell}(P_{\mathrm{T}},P_{\mathrm{S}}) \right| \\
      &\quad \le \frac{\mathfrak{C}_{\mathcal{H}\otimes\mathcal{H}}}{\sqrt{n_{\mathrm{T}}}} + \frac{\mathfrak{C}_{\mathcal{H}\otimes\mathcal{H}}}{\sqrt{n_{\mathrm{S}}}} + \sqrt{\frac{\log\frac{4}{\delta}}{2n_{\mathrm{T}}}} + \sqrt{\frac{\log\frac{4}{\delta}}{2n_{\mathrm{S}}}}
    \end{align*}
under the assumption in \eqref{assump:rademacher}.
\end{corollary}
\begin{proof}
    It simply follows from Theorem~\ref{thm:sipm-deviation},
the assumption in \eqref{assump:rademacher},
and the fact $\mathfrak{R}_{\mu,k}(\ell_{01}\circ(\mathcal{H}\otimes\mathcal{H})) = \frac{1}{2}\mathfrak{R}_{\mu,k}(\mathcal{H}\otimes\mathcal{H})$ for any distribution $\mu$ and $k > 0$~\cite[Lemma 3.1]{Mohri:2012}.
\end{proof}

From this corollary, we see that
the empirical S-disc has the consistency
with convergence rate $\mathcal{O}_p({n_{\mathrm{T}}}^{-1/2} + {n_{\mathrm{S}}}^{-1/2})$ under a mild condition.

\subsection{Generalization Error Bound}
In the previous section, we showed that
S-disc $\varsigma_{\mathcal{H}}^{\ell}(P_{\mathrm{T}},P_{\mathrm{S}})$
can be estimated by
$\varsigma_{\mathcal{H}}^{\ell}(\widehat{P}_{\mathrm{T}},\widehat{P}_{\mathrm{S}})$.
In this section, we give
two bounds on the generalization error
$R_{\mathrm{T}}^{\ell}(h,f_{\mathrm{T}})$
for the target domain
in terms of
$\varsigma_{\mathcal{H}}^{\ell}(P_{\mathrm{T}},P_{\mathrm{S}})$
or
$\varsigma_{\mathcal{H}}^{\ell}(\widehat{P}_{\mathrm{T}},\widehat{P}_{\mathrm{S}})$.

The first bound shows the relationship between target risk $R_{\mathrm{T}}^{\ell}(h,f_{\mathrm{T}})$ and source risk $R_{\mathrm{S}}^{\ell}(h,h_{\mathrm{S}}^{*})$.
\begin{theorem}
\label{thm:generalization-bound}
Assume that $\ell$ obeys the triangular inequality, i.e.,
$
\ell(u,v) \le \ell(u,w) + \ell(w,v), \; \forall u, v, w \in \mathbb{R},
$
such as the 0-1 loss. 
Then, for any hypothesis $h \in \mathcal{H}$, 
\begin{align}
&R_{\mathrm{T}}^{\ell}(h,f_{\mathrm{T}})- R_{\mathrm{T}}^{\ell}(h_{\mathrm{T}}^{*},f_{\mathrm{T}})\notag \\
&\leq  R_{\mathrm{S}}^{\ell}(h,h_{\mathrm{S}}^{*})+ R_{\mathrm{T}}^{\ell}(h_{\mathrm{S}}^*,h_{\mathrm{T}}^{*})+\varsigma_{\mathcal{H}}^{\ell}(P_{\mathrm{T}},P_{\mathrm{S}}). {\label{eq:sipm_bound}}
\end{align}
\end{theorem}
\begin{proof}
Since
$$
R_{\mathrm{T}}^\ell(h,f_{\mathrm{T}}) \le R_{\mathrm{T}}^\ell(h,h_{\mathrm{S}}^*) + R_{\mathrm{T}}^\ell(h_{\mathrm{S}}^*,h_{\mathrm{T}}^*) + R_{\mathrm{T}}^\ell(h_{\mathrm{T}}^*,f_{\mathrm{T}})
$$
holds from the triangular inequality,
we have
\begin{align*}
&R_{\mathrm{T}}^{\ell}(h,f_{\mathrm{T}})- R_{\mathrm{T}}^{\ell}(h_{\mathrm{T}}^{*},f_{\mathrm{T}})\notag\\
&\leq R_{\mathrm{T}}^{\ell}(h,h_{\mathrm{S}}^{*}) +R_{\mathrm{T}}^{\ell}(h_{\mathrm{S}}^{*},h_{\mathrm{T}}^{*}) \\
&\leq  R_{\mathrm{S}}^{\ell}(h,h_{\mathrm{S}}^{*})+ R_{\mathrm{T}}^{\ell}(h_{\mathrm{S}}^{*},h_{\mathrm{T}}^{*})+ \varsigma_{\mathcal{H}}^{\ell}(P_{\mathrm{T}},P_{\mathrm{S}}),
\end{align*}
where the last inequality follows
from the definition of S-disc.
\end{proof}

The LHS of \eqref{eq:sipm_bound} represents the regret arising from the use of hypothesis $h$ instead of $h_{\mathrm{T}}^{*}$ in the target domain. 
Theorem~\ref{thm:generalization-bound} shows that
the regret
$R_{\mathrm{T}}^{\ell}(h,f_{\mathrm{T}})-R_{\mathrm{T}}^{\ell}(h_{\mathrm{T}}^{*},f_{\mathrm{T}})$
is bounded by three terms:
(i) the expected loss with respect to $h_{\mathrm{S}}^{*}$ in the source domain,
(ii) the difference between $h_{\mathrm{T}}^{*}$ and $h_{\mathrm{S}}^{*}$ in the target domain, and
(iii) S-disc between $P_{\mathrm{T}}$ and $P_{\mathrm{S}}$.
Note that if the source and target domains are sufficiently close, we can expect the second term $R_{\mathrm{T}}^{\ell}(h_{\mathrm{S}}^{*},h_{\mathrm{T}}^{*})$ and the third term $\varsigma_{\mathcal{H}}^{\ell}(P_{\mathrm{T}},P_{\mathrm{S}})$ to be small.
This fact indicates that for an appropriate source domain, minimization of estimation error $  R_{\mathrm{S}}^{\ell}(h,h_{\mathrm{S}}^{*})$ in the source domain leads to a better generalization in the target domain.

We can see an advantage of
the generalization error bound
based on S-disc through comparison with
the bound based on
$\mathcal{X}$-disc~\cite[Theorem 8]{mansour2009domain} given by
\begin{align}
&R_{\mathrm{T}}^{\ell}(h,f_{\mathrm{T}}) - R_{\mathrm{T}}^{\ell}(h_{\mathrm{T}}^{*},f_{\mathrm{T}})\notag  \\
&\leq  R_{\mathrm{S}}^{\ell}(h,h_{\mathrm{S}}^{*})+ R_{\mathrm{T}}^{\ell}(h_{\mathrm{T}}^{*},h_{\mathrm{S}}^{*})+ \mathrm{disc}_{\mathcal{H}}^{\ell}(P_{\mathrm{T}},P_{\mathrm{S}}). \label{eq:3}
\end{align}

The upper bound~\eqref{eq:3} using $\mathcal{X}$-disc
has the same form as
the upper bound~\eqref{eq:sipm_bound}
except for the term $\varsigma_{\mathcal{H}}^{\ell}(P_{\mathrm{T}},P_{\mathrm{S}})$.
Since S-disc is never larger than $\mathcal{X}$-disc~(see the inequality~\eqref{eq:1}), S-disc gives a tighter bound than $\mathcal{X}$-disc.

The following theorem shows the generalization error bound for the finite-sample case.
\begin{theorem}
    \label{thm:generalization-bound-with-s-disc}
    When we consider $\ell = \ell_{01}$, for any $h \in \mathcal{H}$ and $\delta \in (0,1)$, with probability at least $1- \delta$,
    \begin{align*}
        &R_{\mathrm{T}}^{\ell}(h,f_{\mathrm{T}})-R_{\mathrm{T}}^{\ell}(h_{\mathrm{T}}^{*},f_{\mathrm{T}})\\
        &\leq \widehat{R}_{\mathrm{S}}^{\ell}(h,h_{\mathrm{S}}^{*})+ R_{\mathrm{T}}^{\ell}(h_{\mathrm{S}}^{*},h_{\mathrm{T}}^{*})+\varsigma_{\mathcal{H}}^{\ell}(\widehat{P}_{\mathrm{T}},\widehat{P}_{\mathrm{S}})\\
        &\quad + \frac{\mathfrak{C}_{\mathcal{H}\otimes\mathcal{H}}}{\sqrt{n_{\mathrm{T}}}} + \frac{\mathfrak{C}_{\mathcal{H}\otimes\mathcal{H}}}{\sqrt{n_{\mathrm{S}}}} + \sqrt{\frac{\log\frac{5}{\delta}}{2n_{\mathrm{T}}}} + 2\sqrt{\frac{\log\frac{5}{\delta}}{2n_{\mathrm{S}}}}
    \end{align*}
under the assumption in \eqref{assump:rademacher}.
\end{theorem}
The proof of this theorem is given in Appendix.

Theorem~\ref{thm:generalization-bound-with-s-disc} tells us that when $n_{\mathrm{S}},n_{\mathrm{T}} \rightarrow \infty$ the following three terms are dominating in the bound of the regret in the target domain $R_{\mathrm{T}}^{\ell}(h,f_{\mathrm{T}})-R_{\mathrm{T}}^{\ell}(h_{\mathrm{T}}^{*},f_{\mathrm{T}})$:
(i) the empirical loss with respect to $h_{\mathrm{S}}^{*}$ in the source domain,
(ii) the difference between $h_{\mathrm{T}}^{*}$ and $h_{\mathrm{S}}^{*}$ in the target domain,
and (iii) S-disc between the two empirical distributions $\varsigma_{\mathcal{H}}^{\ell}(\widehat{P}_{\mathrm{T}},\widehat{P}_{\mathrm{S}})$. Therefore, if $h_{\mathrm{T}}^{*}$ is
sufficiently close to $h_{\mathrm{S}}^{*}$, selecting a good source in terms of S-disc allows us to achieve good target generalization.

\section{Comparison with Existing Discrepancy Measures}
In the previous section,
we showed the consistency of the estimator for S-disc
and derived a generalization error bound of S-disc tighter than $\mathcal{X}$-disc.
In this section,
we first compare these theoretical guarantees of S-disc with those for the existing ones in more detail.
We next discuss their computation cost.
This is also an important aspect
when we apply these discrepancy measures to sentiment analysis~\cite{bhatt2016multi}, adversarial learning~\cite{zhao2018multiple}, and computer vision~\cite{saito2017maximum} for source selection or reweighting of the source data.
In fact, in these applications the discrepancy $d_{\mathcal{H}}$ instead of $\mathcal{X}$-disc is used for ease of computation even though
$d_{\mathcal{H}}$ has no theoretical guarantee on the generalization error.
The results of this section are summarized in Table~\ref{tab:comparison}.

\subsection{Convergence Rates of Discrepancy Estimators}
Here we discuss the consistency and convergence rates of the estimators of discrepancy measures.

The empirical estimator of $d_\mathcal{H}$ is consistent, and its convergence rate is $\mathcal{O}_p(((\log n_{\mathrm{T}})/n_{\mathrm{T}})^{1/2} + ((\log n_{\mathrm{S}})/n_{\mathrm{S}})^{1/2})$~\cite[Lemma 1]{ben2010theory}.
This rate is slower than the rate
$\mathcal{O}_p\left({n_{\mathrm{T}}}^{-1/2}+{n_{\mathrm{S}}}^{-1/2}\right)$
for S-disc with appropriately controlled Rademacher complexities $\mathfrak{R}_{P_{\mathrm{T}},n_{\mathrm{T}}}(\mathcal{H})$ and $\mathfrak{R}_{P_{\mathrm{S}},n_{\mathrm{S}}}(\mathcal{H})$ of the hypothesis class, such as the linear-in-parameter model~\cite[Theorem 4.3]{Mohri:2012}.
Here recall that no generalization error bound is known
in terms of $d_\mathcal{H}$ even though $d_\mathcal{H}$
itself is consistently estimated.

On the other hand, the empirical estimator of $\mathcal{X}$-disc is shown to be consistent ~\cite[Corollary 7]{mansour2009domain} and its convergence rate is $\mathcal{O}_p({n_{\mathrm{T}}}^{-1/2} + {n_{\mathrm{S}}}^{-1/2})$
in the case that the loss function is $\ell_q$ loss, i.e., $\ell_q(y,y')=|y-y'|^q$.
Thus, the derived rate is the same as S-disc
whereas the requirement on the loss function is more restrictive
than the one for S-disc in Theorem~\ref{thm:generalization-bound-with-s-disc}.

Note that the above difference of the theoretical guarantees
does not come from the inherent difference of these estimators.
This is because we adopted the analysis based on the Rademacher complexity,
which has not been well studied in the context of unsupervised domain adaptation.
This is a distribution-dependent complexity measure and less pessimistic compared with the VC-dimension used in the previous work~\cite{ben2010theory}.
In fact, the known guarantees on $d_\mathcal{H}$ and $\mathcal{X}$-disc
can be improved to Propositions~\ref{prop:dh-deviation} and \ref{prop:disc-deviation} given below.

\begin{proposition}
    \label{prop:dh-deviation}
    For any $\delta \in (0, 1)$, with probability at least $1 - \delta$,
    \begin{align*}
        &\left|d_\mathcal{H}(\widehat{P}_{\mathrm{T}},\widehat{P}_{\mathrm{S}}) - d_\mathcal{H}(P_{\mathrm{T}},P_{\mathrm{S}})\right| \\
        &\; \le 2\mathfrak{R}_{P_{\mathrm{T}},n_{\mathrm{T}}}(\mathcal{H}) + 2\mathfrak{R}_{P_{\mathrm{S}},n_{\mathrm{S}}}(\mathcal{H}) + \sqrt{\frac{2\log\frac{4}{\delta}}{n_{\mathrm{T}}}} + \sqrt{\frac{2\log\frac{4}{\delta}}{n_{\mathrm{S}}}}.
    \end{align*}
\end{proposition}
\begin{proposition}
    \label{prop:disc-deviation}
    Assume the loss function $\ell$ is upper bounded by $M > 0$.
    For any $\delta \in (0, 1)$, with probability at least $1 - \delta$,
    \begin{align*}
            &\left|\mathrm{disc}_\mathcal{H}^\ell(\widehat{P}_{\mathrm{T}},\widehat{P}_{\mathrm{S}}) - \mathrm{disc}_\mathcal{H}^\ell(P_{\mathrm{T}},P_{\mathrm{S}})\right| \\
        &\qquad \le 2\mathfrak{R}_{P_{\mathrm{T}},n_{\mathrm{T}}}(\ell\circ(\mathcal{H}\otimes\mathcal{H})) + 2\mathfrak{R}_{P_{\mathrm{S}},n_{\mathrm{S}}}(\ell\circ(\mathcal{H}\otimes\mathcal{H})) \\
        &\qquad \quad + M\sqrt{\frac{\log\frac{4}{\delta}}{2n_{\mathrm{T}}}} + M\sqrt{\frac{\log\frac{4}{\delta}}{2n_{\mathrm{S}}}}.
    \end{align*}
\end{proposition}
The proofs of these propositions are quite similar to that of Theorem~\ref{thm:sipm-deviation} and omitted.

In summary, under the mild assumptions, the convergence rates of the empirical estimators of $d_\mathcal{H}$ and $\mathcal{X}$-disc are $\mathcal{O}_p\left({n_{\mathrm{T}}}^{-1/2}+{n_{\mathrm{S}}}^{-1/2}\right)$ as well as S-disc.

\subsection{Computational Complexity}

Computation of $d_\mathcal{H}$ can be done by the empirical risk minimization~\cite[Lemma 2]{ben2010theory}.
The original form is given with the 0-1 loss, which can be efficiently minimized with a surrogate loss.
When the hinge loss is applied, the minimization can be
carried out with the computation cost $O((n_{\mathrm{T}} + n_{\mathrm{S}})^3)$ by the SMO algorithm~\cite{Platt98sequentialminimal}.

On the other hand, no efficient algorithm is given for the computation of $\mathcal{X}$-disc in the classification setting.\footnote{A computation algorithm of $\mathcal{X}$-disc for the 0-1 loss is given only in the one-dimensional case~\cite[Section 5.2]{mansour2009domain}.}
For a fair comparison, we give a relatively efficient algorithm to compute $\mathcal{X}$-disc~\eqref{def:x-disc} in the classification setting with the hinge loss, based on \emph{semidefinite relaxation}.
Unfortunately, the computational complexity of the relaxed algorithm is
still $O((n_{\mathrm{T}}+n_{\mathrm{S}}+d)^8)$, which is prohibitive compared with the computation of S-disc and $d_\mathcal{H}$.

\begin{figure}[t]
\includegraphics[width=\columnwidth]{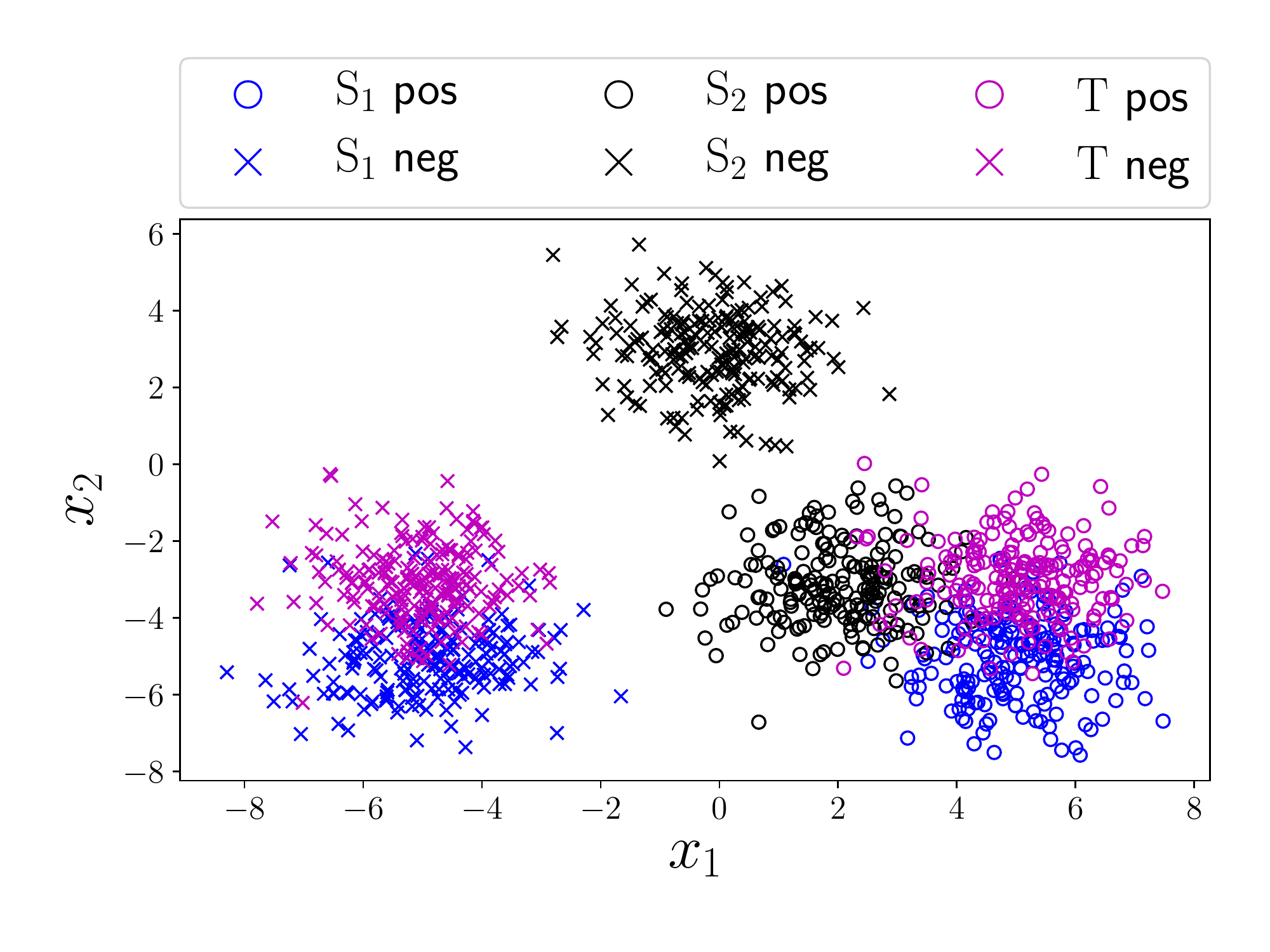}
\caption{2D plots of three domains.}
\label{fig:illustration}
\end{figure}

\section{Experiments}
In this section, we provide the experimental results that illustrate the failure of existing discrepancy measures and the advantage of using S-disc. We illustrate the advantage of S-disc in terms of computation time, the empirical convergence, and the performance on the source selection task.

\subsection{Illustration}
We illustrate the failure of $d_{\mathcal{H}}$, which is the well-known proxy for $\mathcal{X}$-disc. We compared S-disc with $d_{\mathcal{H}}$ in the toy experiment.

We generated $200$ data points per class for each of two sources $\mathrm{S}_1$ and $\mathrm{S}_2$ and target domain as follows:
\begin{align*}
P_{\mathrm{S}_{i}}(X|Y=j)&= \mathcal{N}(\mu_i^{(j)},I) \ (i = 1,2),\\
P_{\mathrm{T}}(X|Y=j)    &= \mathcal{N}(\mu_{\mathrm{T}}^{(j)},I),
\end{align*}
where
\begin{align*}
&\mu_1^{(0)} = (-5, -5), &\mu_2^{(0)} = (0, 3),&\quad \mu_{\mathrm{T}}^{(0)} = (-5, -3),\\
&\mu_1^{(1)} = (5, -5), &\mu_2^{(1)} = (2, -3),&\quad \mu_{\mathrm{T}}^{(1)} = (5, -3).
\end{align*}

In this experiment, we used the support vector machine with a linear kernel\footnote{We used scikit-learn~\cite{pedregosa2011scikit} to implement it.}.
For these data, we obtained the following results:
\begin{align*}
\varsigma_{\mathcal{H}}^{\ell}(\widehat{P}_{\mathrm{T}},\widehat{P}_{\mathrm{S}_1})= 0.27,
&\quad\varsigma_{\mathcal{H}}^{\ell}(\widehat{P}_{\mathrm{T}},\widehat{P}_{\mathrm{S}_2})= 0.49,\\
d_{\mathcal{H}}(\widehat{P}_{\mathrm{T}},\widehat{P}_{\mathrm{S}_1})= 0.69,
&\quad d_{\mathcal{H}}(\widehat{P}_{\mathrm{T}},\widehat{P}_{\mathrm{S}_2})= 0.49.
\end{align*}

These values indicate that while $d_{\mathcal{H}}$ regards $\mathrm{S}_2$ as the better source, S-disc regards $\mathrm{S}_1$ as the better source for target domain.

In this example, the loss calculated on the target domain of the classifier trained on $\mathrm{S}_1$ is 0.0 and the loss of the classifier trained on $\mathrm{S}_2$ is 0.49.
This implies that S-disc is the better discrepancy measure to measure the quality of source domains for a better generalization in a given target domain.

Intuitively, $d_\mathcal{H}$ is a heuristic measure based on a classifier separating the source and target domains without considering labels in the source domain.  Once the supports of the input distributions are not highly overlapped between the source and target domains
, $d_\mathcal{H}$ may immediately regard that domains are totally different from each other even if the risk minimizers are highly similar between these domains, which resulted in the failure of $d_\mathcal{H}$. On the other hand, S-disc can prevent such a problem by taking labels from the source domain into account as illustrated in the toy experiments and justified in Theorems~\ref{thm:generalization-bound} and \ref{thm:generalization-bound-with-s-disc}.

\subsection{Comparison of Computation Time}
We compared the computation time to estimate S-disc, $d_{\mathcal{H}}$, and $\mathcal{X}$-disc.
We used 2-dimensional 200 examples for both source and target domains. Each domain consists of 100 positive examples and 100 negative examples. 
For the computation of $\mathcal{X}$-disc, we used the relaxed algorithm\footnote{We used picos (\url{https://picos-api.gitlab.io/picos/}) for the SDP solver with cvxopt (\url{https://cvxopt.org/}) backend.} shown in Appendix. For both $d_{\mathcal{H}}$ and S-disc, we used the support vector machine with a linear kernel for the computation of S-disc and $d_\mathcal{H}$. 
The simulation was run on 2.8GHz Intel\textsuperscript{\textregistered} Core i7.
The results shown in Figure~\ref{fig:time} demonstrate that the computation time of $\mathcal{X}$-disc is prohibitive while both S-disc and $d_{\mathcal{H}}$ are feasible to compute as suggested in Table~\ref{tab:comparison}.

\begin{figure}[t]
    \includegraphics[width=\columnwidth]{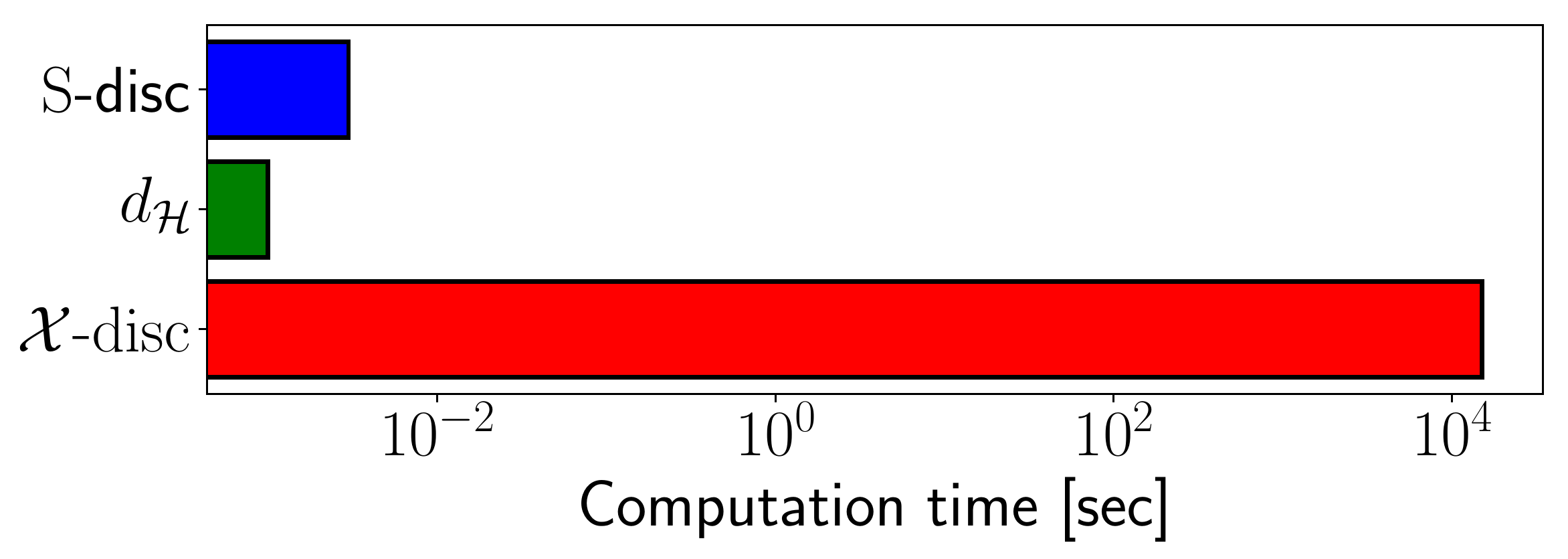}
    \caption{Comparison of computation time in the log scale.}
    \label{fig:time}
\end{figure}

\subsection{Empirical Convergence}
We compared the empirical convergence of S-disc and $d_\mathcal{H}$ based on logistic regression implemented with scikit-learn with default parameters~\cite{pedregosa2011scikit}. Note that $\mathcal{X}$-disc is not used due to computational intractability.
Here, we used MNIST~\cite{lecun2010mnist} dataset, and considered a binary classification task to separate odd and even digits.
We defined the source and target domains as follows:
\begin{itemize}
\item Source domain $\mathrm{S}_1$: MNIST,
\item Source domain $\mathrm{S}_2$: MNIST from zero to seven,
\item Target domain: MNIST.
\end{itemize}
The number of examples was ranged from \{1000, 2000, \dots, 20000\}  for each domain. Note that while examples from the source domain $\mathrm{S}_1$ are drawn from the same distribution of the target domain, the sample of the source domain $\mathrm{S}_2$ is affected by sample selection bias~\cite{cortes2008sample}.  Figure~\ref{fig:exp1} shows the empirical convergence of the estimator of both discrepancy measures.
It is observed that both discrepancy measures indicate that $\mathrm{S}_1$ is a better source than $\mathrm{S}_2$ for target domain.
However, since the true value of the discrepancy of $\mathrm{S}_1$ from the target domain is supposed to be zero, we can observe that the $d_\mathcal{H}$ will be converged much more slowly than S-disc or converged to a non-zero value.

\begin{figure}[t]
\centering
\includegraphics[width=\columnwidth]{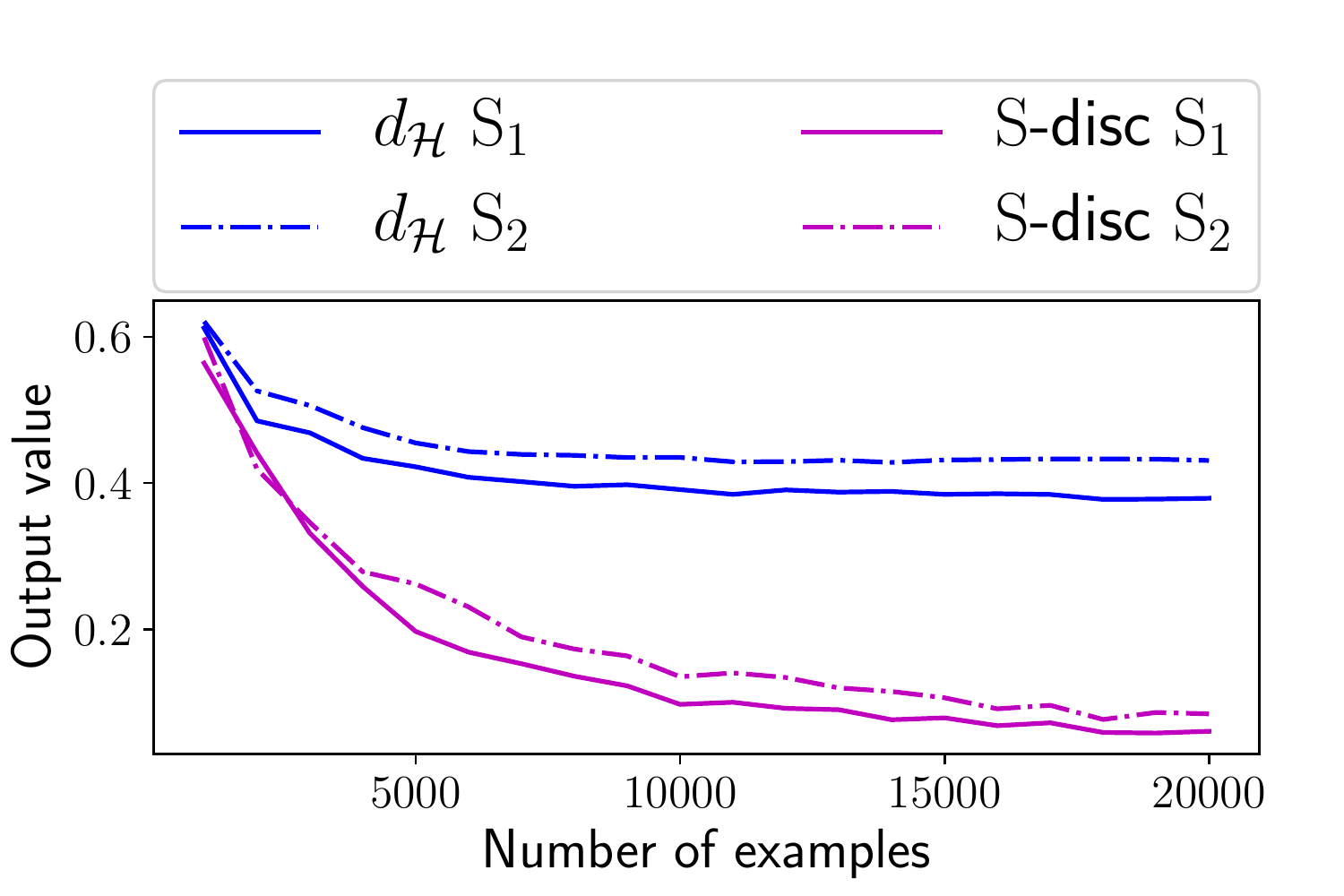}
\caption{Empirical convergence of $d_\mathcal{H}$ and S-disc.}
\label{fig:exp1}
\end{figure}

\subsection{Source Selection}
We compared the performance in the source selection task between S-disc and $d_{\mathcal{H}}$.
We used the source domains and the target domain as follows:
\begin{itemize}
\item Clean source domains: Five grayscale MNIST-M~\cite{ganin2016domain},
\item Noisy source domains: Five grayscale MNIST-M corrupted by Gaussian random noise,
\item Target domain: MNIST.
\end{itemize}
The MNIST-M dataset is known to be useful for the domain adaptation task when the target domain is MNIST.
The task of each domain is to classify between even and odd digits and logistic regression with default parameters was used for computing S-disc and $d_\mathcal{H}$.
The objective of this source selection task is to correctly rank five clean source domains over noisy domains.
We ranked each source by computing discrepancy values between the target domain and each source and rank them in ascending order.
The score was calculated by counting how many clean sources are ranked in the first five sources.
We varied the number of examples from \{200, 400, \dots, 4000\} for each domain.
The Gaussian noise with standard deviation $\epsilon = {30, 40, 50}$ were added and clipped to force the value to between $0-255$. For each number of examples per class, the experiments were conducted $15$ times and the average score was used.

Figure~\ref{exp2} shows the performance of each discrepancy measure with different noise rates.
As the number of examples increased, S-disc achieved a better performance.
In contrast, $d_\mathcal{H}$ cannot distinguish between noisy and clean source domains.
In fact, $d_\mathcal{H}$ always returned one, which indicates that MNIST-M is unrelated to MNIST. Unlike the previous experiment on empirical convergence, the difference between the two domains may be harder to identify since the source domains and the target domain are not exactly the same (MNIST vs MNIST-M). As a result, our experiment demonstrates the failure of $d_\mathcal{H}$ in the source selection task and suggests the advantage of S-disc in this task.

\begin{figure}[t]
    \includegraphics[width=\columnwidth]{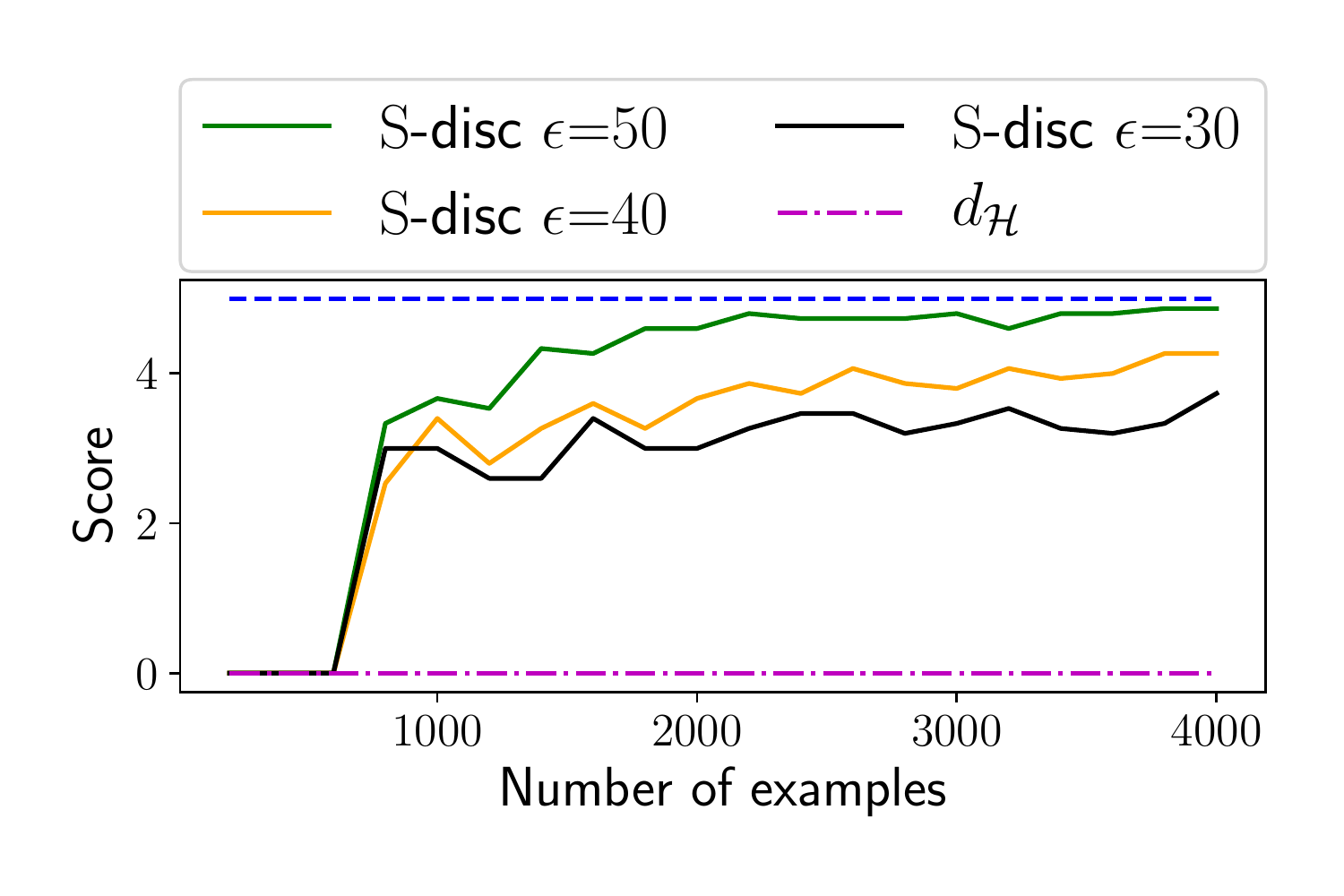}
    \caption{Source selection performance with varying noise rates.  Blue dotted line denotes the maximum score (five).\\\mbox{}}
    \label{exp2}
\end{figure}

\section{Conclusion}
We proposed a novel discrepancy measure for unsupervised domain adaptation called source-guided discrepancy (S-disc). We provided a computationally efficient algorithm for the estimation of S-disc with respect to the 0-1 loss, and also established the consistency of the estimator with the convergence rate. Moreover, we derived a generalization bound based on S-disc, which is tighter than that of $\mathcal{X}$-disc. Finally, we demonstrated the failure of existing discrepancy measures and the advantages of S-disc through experiments. 

\section{Acknowledgements}
We thank Ikko Yamane, Futoshi Futami, and Kento Nozawa for the useful discussion.
NC was supported by MEXT scholarship. JH was supported by KAKENHI 17H00757, IS was supported by JST CREST Grant Number JPMJCR17A1, and MS was supported by KAKENHI 17H01760.
\nocite{kuroki2019aaai}
\bibliography{SIPM}
\bibliographystyle{aaai}

\appendix
\allowdisplaybreaks
\section{Proof of Theorem~\ref{thm:sipm-estimation}}
\begin{proof}
First, the following equality holds:
\begin{align*} 
&\widehat{R}_{\mathrm{T}}^{\ell_{01}}(h,h_{\mathrm{S}}^*) - \widehat{R}_{\mathrm{S}}^{\ell_{01}}(h,h_{\mathrm{S}}^*)\\
&= (1- \widehat{R}_{\mathrm{T}}^{\ell_{01}}(h, -h_{\mathrm{S}}^*)) - \widehat{R}_{\mathrm{S}}^{\ell_{01}}(h,h_{\mathrm{S}}^*) \\
&=1 - \left(\widehat{R}_{\mathrm{S}}^{\ell_{01}}(h,h_{\mathrm{S}}^*) + \widehat{R}_{\mathrm{T}}^{\ell_{01}}(h,-h_{\mathrm{S}}^*)\right).
\end{align*}
We also have,
\begin{align*}
& \widehat{R}_{\mathrm{T}}^{\ell_{01}}(h,h_{\mathrm{S}}^*) - \widehat{R}_{\mathrm{S}}^{\ell_{01}}(h,h_{\mathrm{S}}^*)\\
&=  1- \widehat{R}_{\mathrm{T}}^{\ell_{01}}(-h,h_{\mathrm{S}}^*) - \left(1- \widehat{R}_{\mathrm{S}}^{\ell_{01}}(-h,h_{\mathrm{S}}^*)\right)\\
&= -\left( \widehat{R}_{\mathrm{T}}^{\ell_{01}}(-h,h_{\mathrm{S}}^*) - \widehat{R}_{\mathrm{S}}^{\ell_{01}}(-h,h_{\mathrm{S}}^*)\right).
\end{align*}
Then, for a symmetric hypothesis class $\mathcal{H}$, the following equality holds:
\begin{align*}
\varsigma_{\mathcal{H}}^{\ell}(\widehat{P}_{\mathrm{T}},\widehat{P}_{\mathrm{S}}) 
&= \supremum_{h \in \mathcal{H}} \left|\widehat{R}_{\mathrm{S}}^{\ell_{01}}(h,h_{\mathrm{S}}^*) - \widehat{R}_{\mathrm{T}}^{\ell_{01}}(h,h_{\mathrm{S}}^*)\right|\\
&= \supremum_{h \in \mathcal{H}} \left( \widehat{R}_{\mathrm{T}}^{\ell_{01}}(h,h_{\mathrm{S}}^*) - \widehat{R}_{\mathrm{S}}^{\ell_{01}}(h,h_{\mathrm{S}}^*)\right)\\
&= \supremum_{h \in \mathcal{H}}\left(1 - \left(\widehat{R}_{\mathrm{S}}^{\ell_{01}}(h,h_{\mathrm{S}}^*) + \widehat{R}_{\mathrm{T}}^{\ell_{01}}(h,-h_{\mathrm{S}}^*)\right)\right)\\
&= 1 - \min_{h \in \mathcal{H}}\left(\widehat{R}_{\mathrm{S}}^{\ell_{01}}(h,h_{\mathrm{S}}^*) + \widehat{R}_{\mathrm{T}}^{\ell_{01}}(h,-h_{\mathrm{S}}^*)\right).
\end{align*}
Next, by the definition of $J_{\ell_{01}}$, we have
\begin{align*}
J_{\ell_{01}}(h)=\widehat{R}_{\mathrm{S}}^{\ell_{01}}(h,h_{\mathrm{S}}^*) + \widehat{R}_{\mathrm{T}}^{\ell_{01}}(h,-h_{\mathrm{S}}^*).
\end{align*}
Thus, 
\begin{align*}
\varsigma_{\mathcal{H}}^{\ell_{01}}(\widehat{P}_{\mathrm{T}},\widehat{P}_{\mathrm{S}})= 1 - \min_{h \in \mathcal{H}}J_{\ell_{01}}(h).
\end{align*}
\end{proof}

\section{Proof of Theorem~\ref{thm:sipm-deviation}}
For simplicity, we use the following notation:
\begin{itemize}
    \item $P_{\mathrm{T}}f = \int f\mathrm{d}P_{\mathrm{T}}$,
    \item $P_{\mathrm{S}}f = \int f\mathrm{d}P_{\mathrm{S}}$,
    \item $\widehat{P}_{\mathrm{T}}f = {n_{\mathrm{T}}}^{-1}\sum_{i=1}^{n_{\mathrm{T}}}f(x_i^{\mathrm{T}})$,
    \item $\widehat{P}_{\mathrm{S}}f = {n_{\mathrm{S}}}^{-1}\sum_{j=1}^{n_{\mathrm{S}}}f(x_j^{\mathrm{S}})$,
\end{itemize}
where $f: \mathcal{X} \rightarrow \mathbb{R}$ is a real-valued function.
For example, $P\ell(h,h')$ means $\int \ell(h(x),h'(x))\mathrm{d}P$.

In order to derive the convergence rate of this theorem, we apply the \emph{uniform deviation bound} explained in the next lemma.
\begin{lemma}
    \label{lem:uniform-deviation-bound}
    Let $B > 0$ be an arbitrary constant, $\mathcal{F} := \{f: \mathcal{X} \rightarrow [0, B]\}$ be a class of bounded functions, and $\{x_i\}_{i=1}^k \subset \mathcal{X}^k$ be an i.i.d. sample drawn from $p$.
    Then, for any $\delta \in (0, 1)$, with probability at least $1-\delta$,
    $$
        \sup_{f\in\mathcal{F}}\left|\frac{1}{k}\sum_{i=1}^kf(x_i) - \int f(x)\mathrm{d}p\right| \le 2\mathfrak{R}_{p,k}(\mathcal{F}) + B\sqrt{\frac{\log\frac{2}{\delta}}{2k}}
        .
    $$
\end{lemma}
Lemma~\ref{lem:uniform-deviation-bound} immediately follows from Theorem~3.1 in \citet{Mohri:2012}.
To be exact, their result only leads to the one-sided version of Lemma~\ref{lem:uniform-deviation-bound}, but we can easily extend it by applying McDiarmid's inequality~\cite{McDiarmid:1989} twice.
We can obtain the convergence rate by controlling the order of the Rademacher complexity $\mathfrak{R}_{p,k}(\mathcal{F})$ appropriately, e.g., the assumption in~\eqref{assump:rademacher}.
From now on, we prove Theorem~\ref{thm:sipm-deviation}.

\begin{proof}[Proof of Theorem~\ref{thm:sipm-deviation}]
First, we show that the deviation $|\varsigma_{\mathcal{H}}^{\ell}(\widehat{P}_{\mathrm{T}},\widehat{P}_{\mathrm{S}}) - \varsigma_{\mathcal{H}}^{\ell}(P_{\mathrm{T}},P_{\mathrm{S}})|$ can be bounded from above by the summation of the source and target deviations:
\begin{align*}
    &\left|\varsigma_{\mathcal{H}}^{\ell}(\widehat{P}_{\mathrm{T}},\widehat{P}_{\mathrm{S}}) - \varsigma_{\mathcal{H}}^{\ell}(P_{\mathrm{T}},P_{\mathrm{S}})\right|
    \\
    &\; = \left|\sup_{h\in\mathcal{H}}\left|\widehat{P}_{\mathrm{T}}\ell\left(h,h_{\mathrm{S}}^*\right)-\widehat{P}_{\mathrm{S}}\ell\left(h,h_{\mathrm{S}}^*\right)\right| \right. \\
    &\qquad \left. - \sup_{h\in\mathcal{H}}\left|P_{\mathrm{T}}\ell\left(h,h_{\mathrm{S}}^*\right)-P_{\mathrm{S}}\ell\left(h,h_{\mathrm{S}}^*\right)\right|\right|
    \\
    &\; \le \sup_{h\in\mathcal{H}}\left|\left|(\widehat{P}_{\mathrm{T}}-\widehat{P}_{\mathrm{S}})\ell\left(h,h_{\mathrm{S}}^*\right)\right|-\left|(P_{\mathrm{T}}-P_{\mathrm{S}})\ell\left(h,h_{\mathrm{S}}^*\right)\right|\right|
    \\
    &\; \le \sup_{h,h'\in\mathcal{H}}\left|\left|(\widehat{P}_{\mathrm{T}}-\widehat{P}_{\mathrm{S}})\ell(h,h')\right|-\left|(P_{\mathrm{T}}-P_{\mathrm{S}})\ell(h,h')\right|\right|
    \\
    &\; \le \sup_{h,h'\in\mathcal{H}}\left|(\widehat{P}_{\mathrm{T}}-P_{\mathrm{T}})\ell(h,h')-(\widehat{P}_{\mathrm{S}}-P_{\mathrm{S}})\ell(h,h')\right|
    \\
    &\; \le \sup_{h,h'\in\mathcal{H}}\left|(\widehat{P}_{\mathrm{T}}-P_{\mathrm{T}})\ell(h,h')\right| + \sup_{h,h'\in\mathcal{H}}\left|(\widehat{P}_{\mathrm{S}}-P_{\mathrm{S}})\ell(h,h')\right|
    ,
\end{align*}
where the third and fourth inequalities follow from the triangular inequality.
Here, the deviations in the last line can be bounded by Lemma~\ref{lem:uniform-deviation-bound} as follows:
\begin{align*}
    &\sup_{h,h\in\mathcal{H}}\left|(\widehat{P}_{\mathrm{T}}-P_{\mathrm{T}})\ell(h,h')\right| \\
    &\qquad \le 2\mathfrak{R}_{P_{\mathrm{T}},n_{\mathrm{T}}}(\ell\circ(\mathcal{H}\otimes\mathcal{H}))+M\sqrt{\frac{\log\frac{4}{\delta}}{2n_{\mathrm{T}}}},
    \\
    &\sup_{h,h\in\mathcal{H}}\left|(\widehat{P}_{\mathrm{S}}-P_{\mathrm{S}})\ell(h,h')\right| \\
    &\qquad \le 2\mathfrak{R}_{P_{\mathrm{S}},n_{\mathrm{S}}}(\ell\circ(\mathcal{H}\otimes\mathcal{H}))+M\sqrt{\frac{\log\frac{4}{\delta}}{2n_{\mathrm{S}}}},
\end{align*}
where each of them holds with probability at least $1-\tfrac{\delta}{2}$.
This concludes the proof.
\end{proof}

\section{Proof of Lemma~\ref{lem:rademacher-bound}}
\begin{proof}
Recall that we use $\phi_i$ to denote $\phi(x_i)$ for simplicity.
The Rademacher complexity $\mathfrak{R}_{P_{\mathrm{T}},n_{\mathrm{T}}}(\mathcal{H}\otimes\mathcal{H})$ can be bounded by using the Frobenius norm $\|\cdot\|_\mathbb{F}$.
\begin{align}
    \mathfrak{R}_{P_{\mathrm{T}},n_{\mathrm{T}}}(\mathcal{H}\otimes\mathcal{H})
    &= \frac{1}{n_{\mathrm{T}}}\mathbb{E}_{x,\sigma}\left[\sup_{w,w'\in\mathbb{R}^d}\sum_{i=1}^{n_{\mathrm{T}}}\sigma_iw^\top \phi_i\cdot w'^\top \phi_i\right] \nonumber
    \\
    &= \frac{1}{n_{\mathrm{T}}}\mathbb{E}_{x,\sigma}\left[\sup_{w,w'\in\mathbb{R}^d}w^\top\left(\sum_{i=1}^{n_{\mathrm{T}}}\sigma_i\phi_i\phi_i^\top w'\right)\right] \nonumber
    \\
    &\le \frac{\Lambda}{n_{\mathrm{T}}}\mathbb{E}_{x,\sigma}\left[\sup_{w,w'\in\mathbb{R}^d}\left\|\left(\sum_{i=1}^{n_{\mathrm{T}}}\sigma_i\phi_i\phi_i^\top\right)w\right\|\right] \nonumber \\
    &\le \frac{\Lambda^2}{n_{\mathrm{T}}}\mathbb{E}_{x,\sigma}\left[\left\|\sum_{i=1}^{n_{\mathrm{T}}}\sigma_i\phi_i\phi_i^\top\right\|_\mathbb{F}\right] \nonumber
    \\
    &\le \frac{\Lambda^2}{n_{\mathrm{T}}}\sqrt{\mathbb{E}_{x,\sigma}\left[\left\|\sum_{i=1}^{n_{\mathrm{T}}}\sigma_i\phi_i\phi_i^\top\right\|_\mathbb{F}^2\right]} \nonumber
    ,
\end{align}
where the first and the second inequalities use the Cauchy-Schwartz inequality,
and the last inequality uses the Jensen's inequality.
Then,
\begin{align}
    \mathbb{E}_{x,\sigma}\left[\left\|\sum_{i=1}^{n_{\mathrm{T}}}\sigma_i\phi_i\phi_i^\top\right\|_\mathbb{F}^2\right]
    &= \mathbb{E}_{x,\sigma}\left[\sum_{i=1,j=1}^{n_{\mathrm{T}}}\sigma_i\sigma_j(\phi_i\phi_i^\top)(\phi_j\phi_j^\top)\right] \nonumber
    \\
    &= \mathbb{E}_{x}\left[\sum_{i=1}^{n_{\mathrm{T}}}\|\phi_i\phi_i^\top\|_\mathbb{F}^2\right]
    \nonumber
    ,
\end{align}
where the last equality follows from $\mathbb{E}_\sigma[\sigma_i\sigma_j]=\mathbb{E}_\sigma[\sigma_i]\mathbb{E}_\sigma[\sigma_j]=0$ for $i \ne j$.
Here,

\begin{align}
    \|\phi_i\phi_i^\top\|_\mathbb{F}^2
    &= \mathrm{tr}\left((\phi_i\phi_i^\top)^\top(\phi_i\phi_i^\top)\right) \nonumber
    \\
    &= \mathrm{tr}\left(\phi_i\phi_i^\top\phi_i\phi_i^\top\right) \nonumber
    \\
    &= \mathrm{tr}\left(\phi_i^\top\phi_i\phi_i^\top\phi_i\right) \nonumber
    \\
    &= \left(\|\phi_i\|^2\right)^2\nonumber
    \\
    &\le D_\phi^4, \nonumber
\end{align}
where the last inequality follows from the assumption $\|\phi\|_\infty \le D_\phi$.
Combining the above results, we have
\begin{align*}
    \mathfrak{R}_{P_{\mathrm{T}},n_{\mathrm{T}}}(\mathcal{H}\otimes\mathcal{H})
    &\le \frac{\Lambda^2}{n_{\mathrm{T}}}\sqrt{\mathbb{E}_x\left[\sum_{i=1}^{n_{\mathrm{T}}}\|\phi_i\phi_i^\top\|_\mathbb{F}^2\right]}
    \\
    &\le \frac{\Lambda^2}{n_{\mathrm{T}}}\sqrt{n_{\mathrm{T}}\cdot D_\phi^4}
    \\
    &= \frac{\Lambda^2D_\phi^2}{\sqrt{n_{\mathrm{T}}}},
\end{align*}
which concludes the proof.
\end{proof}

\section{Proof of Theorem~\ref{thm:generalization-bound-with-s-disc}}

\begin{proof}
    Fix one $h \in \mathcal{H}$.
    First, $\widehat{R}_{\mathrm{S}}^\ell(h,h_{\mathrm{S}}^*)=\frac{1}{n_{\mathrm{S}}}\sum_{i=1}^{n_{\mathrm{S}}}\ell(h(x_i^{\mathrm{S}}),h_{\mathrm{S}}^*(x_i^{\mathrm{S}}))$ has bounded differences, that is, for two samples $S=(x_1^{\mathrm{S}},\dots,x_i^{\mathrm{S}},\dots,x_{n_{\mathrm{S}}}^{\mathrm{S}})$ and $S'=(x_1^{\mathrm{S}},\dots,{x_i^{\mathrm{S}}}',\dots,x_{n_{\mathrm{S}}}^{\mathrm{S}})$ which differ only at the $i$-th element,
    $$
        \left|\widehat{R}_{\mathrm{S}}(h,h_{\mathrm{S}}^*) - \widehat{R}_{S'}(h,h_{\mathrm{S}}^*)\right| \le \frac{1}{n_{\mathrm{S}}}.
    $$
    Note that we consider $\ell=\ell_{01}$ here and it is bounded by $1$ from above.
    Thus, from McDiarmid's inequailty~\cite{McDiarmid:1989}, for any $\delta \in (0,1)$, with probability at least $1-\frac{\delta}{5}$,
    \begin{align}
        R_{\mathrm{S}}^\ell(h,h_{\mathrm{S}}^*) \le \widehat{R}_{\mathrm{S}}^\ell(h,h_{\mathrm{S}}^*) + \sqrt{\frac{\log\frac{5}{\delta}}{2n_{\mathrm{S}}}}.
        \label{eq:mcdiarmid-result}
    \end{align}
    In addition, from Corollary~\ref{cor:01-sipm-rate},
    \begin{align}
        &\varsigma_{\mathcal{H}}^\ell(P_{\mathrm{T}},P_{\mathrm{S}}) \nonumber \\
        &\quad \le \varsigma_{\mathcal{H}}^\ell(\widehat{P}_{\mathrm{T}},\widehat{P}_{\mathrm{S}}) + \mathfrak{R}_{P_{\mathrm{T}},n_{\mathrm{T}}}(\mathcal{H}\otimes\mathcal{H}) + \mathfrak{R}_{P_{\mathrm{S}},n_{\mathrm{S}}}(\mathcal{H}\otimes\mathcal{H}) \nonumber \\
        &\qquad + \sqrt{\frac{\log\frac{5}{\delta}}{2n_{\mathrm{T}}}} + \sqrt{\frac{\log\frac{5}{\delta}}{2n_{\mathrm{S}}}},
        \label{eq:01-s-disc-rate-result}
    \end{align}
    with probability at least $1 - \frac{4\delta}{5}$.
    Under the assumption in~\eqref{assump:rademacher}, combining Eqs.~\eqref{eq:mcdiarmid-result} and~\eqref{eq:01-s-disc-rate-result} together with Theorem~\ref{thm:generalization-bound} gives the conclusion.
\end{proof}

\section{Semidefinite Relaxation of $\mathcal{X}$-disc Computation}

First, we assume the hypothesis class $\mathcal{H}=\{\bm{\beta}^\top\bm{\phi}(\bm{x}) \mid \bm{\beta} \in \mathbb{R}^d\}$, where $\bm{\phi}: \mathcal{X} \rightarrow \mathbb{R}^d$ are basis functions.
Let $N = n_{\mathrm{T}} + n_{\mathrm{S}}$ and $\bm{a} \in \mathbb{R}^N$ be a vector with values
\begin{align*}
    a_i = \begin{cases}
        \tfrac{1}{n_{\mathrm{T}}} & 1 \le i \le n_{\mathrm{T}}, \\
        -\tfrac{1}{n_{\mathrm{S}}} & n_{\mathrm{T}}+1 \le i \le N,
    \end{cases}
\end{align*}
and $\bm{x} \in \mathcal{X}^N$ be
\begin{align*}
    x_i = \begin{cases}
        x_i^{\mathrm{S}} & 1 \le i \le n_{\mathrm{T}}, \\
        x_{i-{n_{\mathrm{T}}}}^{\mathrm{T}} & n_{\mathrm{T}}+1 \le i \le N.
    \end{cases}
\end{align*}
We can rewrite the $\mathcal{X}$-disc computation problem with the hinge loss in the following way:
\begin{align}
    \max_{\substack{\bm{\beta}, \bm{\beta}' \in \mathbb{R}^d \\ \bm{\xi} \in \mathbb{R}^N}} \bm{a}^\top\bm{\xi} \;\; \text{s.t.} \;
    \begin{cases}
        \xi_i \ge 0 & (1 \le i \le N),\\
        \xi_i \ge 1 - \bm{\beta}^\top(\bm{\phi}_i\bm{\phi}_i^\top)\bm{\beta}' & (1 \le i \le N) ,
    \end{cases}
    \label{eq:discrepancy-estimation}
\end{align}
where $\bm{\phi}_i$ is an abbreviation of $\bm{\phi}(x_i)$.
The problem~\eqref{eq:discrepancy-estimation} can be rewritten in the following way:
\begin{align}
    \max_{\bm{z}\in\mathbb{R}^{N+2d}} \bm{c}^\top\bm{z} \;\; \text{s.t.} \;
    \begin{cases}
        \bm{f}_i^\top\bm{z} \ge 0 & (1 \le i \le N),\\
        \bm{z}^\top\Phi_i\bm{z} + \bm{f}_i^\top\bm{z} \ge 1 & (1 \le i \le N),
    \end{cases}
    \label{eq:nonconvex-qcqp}
\end{align}
where
$\bm{z}=[\bm{\xi}^\top\;\bm{\beta}^\top\;\bm{\beta}'^\top]^\top$,
$\bm{c}=[\bm{a}^\top\;\bm{0}_{2d}^\top]^\top$,
$\bm{f}_i \in \mathbb{R}^{N+2d}$ are vectors with elements $f_{ij} = \delta_{ij}$ (for $1 \le i \le N$)\footnote{$\delta_{ij}$ is the Kronecker delta, that is, $\delta_{ij}=1$ if $i=j$, otherwise $\delta_{ij}=0$.}, and
$$
\Phi_i = \frac{1}{2}\left[\begin{matrix}
    O_{N \times N} & O_{N \times d} & O_{N \times d} \\
    O_{d \times N} & O_{d \times d} & \bm{\phi}_i\bm{\phi}_i^\top \\
    O_{d \times N} & \bm{\phi}_i\bm{\phi}_i^\top & O_{d \times d}
\end{matrix}\right],
$$
where $O_{k\times l}$ is $k\times l$ zero matrix.
Since the problem~\eqref{eq:nonconvex-qcqp} is a nonconvex quadratically constrained quadratic programming\footnote{The latter constraints $\bm{z}^\top\Phi_i\bm{z}+\bm{f}_i^\top\bm{z} \ge 1$ form nonconvex feasible regions, because each $\Phi_i$ is positive semidefinite.}, it is an NP-hard problem in general, and intractable.
\emph{Semidefinite relaxation~(SDR)}~\cite{fujie1997semidefinite,kim2001second} helps to make the problem~\eqref{eq:nonconvex-qcqp} tractable at the sacrifice of the exact solution, by introducing new optimization variables $Z \in \mathbb{S}^{N+2d}$:
\begin{align}
    \max_{\substack{\bm{z}\in\mathbb{R}^{N+2d} \\ Z\in\mathbb{S}^{N+2d}}} \bm{c}^\top\bm{z} \;\; \text{s.t.} \;
    \begin{cases}
        \bm{f}_i^\top\bm{z} \ge 0 & (1 \le i \le N),\\
        \langle \Phi_i, Z \rangle_{\mathbb{F}} + \bm{f}_i^\top\bm{z} \ge 1 & (1 \le i \le N),\\
        Z - \bm{z}\bm{z}^\top \succeq 0,
    \end{cases}
    \label{eq:sdp-relaxation}
\end{align}
where
$\mathbb{S}^{k}$ is the set of $k \times k$ symmetric matrices,
$\langle A,B \rangle_{\mathbb{F}} = \sum_{i,j}A_{ij}B_{ij}$ is the Frobenius inner product,
and $A \succeq 0$ means that the matrix $A$ is positive semidefinite.
Since the problem~\eqref{eq:sdp-relaxation} is \emph{semidefinite programming (SDP)}, it can be solved in polynomial time.
If the ellipsoid method is applied, the computational complexity of the problem~\eqref{eq:sdp-relaxation} is $O((n_{\mathrm{T}} + n_{\mathrm{S}} + d)^8)$~\cite{bubeck2015convex}.
\end{document}